\documentclass[conference]{IEEEtran}
\IEEEoverridecommandlockouts
\usepackage{cite}
\usepackage{amsmath,amssymb,amsfonts}
\usepackage{algorithmic}
\usepackage{graphicx}
\usepackage{textcomp}
\usepackage{xcolor}
\def\BibTeX{{\rm B\kern-.05em{\sc i\kern-.025em b}\kern-.08em
    T\kern-.1667em\lower.7ex\hbox{E}\kern-.125emX}}

\usepackage{times}
\usepackage{soul}
\usepackage{url}
\usepackage[hidelinks]{hyperref}
\usepackage[utf8]{inputenc}
\usepackage[small]{caption}
\usepackage{graphicx}
\usepackage{amsmath,amsfonts}
\usepackage{amsthm}
\usepackage{booktabs}
\usepackage{multirow}
\usepackage{color}
\usepackage{subcaption}
\usepackage[ruled,vlined, linesnumbered]{algorithm2e}
\newtheorem{theorem}{Theorem}

\newtheorem{proposition}[theorem]{Proposition}

\begin{document}

\title{LEGATO: A LayerwisE Gradient AggregaTiOn Algorithm for Mitigating Byzantine Attacks in Federated Learning
}

\author{\IEEEauthorblockN{Kamala Varma}
\IEEEauthorblockA{\textit{University of Maryland,} \\
\textit{College Park\textsuperscript{*}} \\
\textit{kvarma@umd.edu}}
\and
\IEEEauthorblockN{Yi Zhou}
\IEEEauthorblockA{\textit{IBM Research,} \\
\textit{Almaden Research Center} \\
\textit{yi.zhou@ibm.com}}
\and
\IEEEauthorblockN{Nathalie Baracaldo}
\IEEEauthorblockA{\textit{IBM Research,} \\
\textit{Almaden Research Center} \\
\textit{baracald@us.ibm.com}}
\and
\IEEEauthorblockN{Ali Anwar}
\IEEEauthorblockA{\textit{IBM Research,} \\
\textit{Almaden Research Center} \\
\textit{ali.anwar2@ibm.com}}
}

\maketitle
\begingroup\renewcommand\thefootnote{*}
\footnotetext{This work was completed when the first author was doing an internship at IBM Almaden Research Center.}
\endgroup

\begin{abstract}
Federated learning has arisen as a mechanism to allow multiple participants to collaboratively train a model without sharing their data. In these settings, participants (workers) may not trust each other fully;
for instance, a set of competitors may collaboratively train a machine learning model to detect fraud.  The workers provide local gradients that a central server uses to update a global model.  This global model can be corrupted when Byzantine workers send malicious gradients, which necessitates robust methods for aggregating gradients that mitigate the adverse effects of Byzantine inputs.  Existing robust aggregation algorithms are often computationally expensive and only effective under strict assumptions. In this paper, we introduce LayerwisE Gradient AggregatTiOn (LEGATO), an aggregation algorithm that is, by contrast, scalable and generalizable.  Informed by a study of layer-specific responses of gradients to Byzantine attacks, LEGATO employs a dynamic gradient reweighing scheme that is novel in its treatment of gradients based on layer-specific robustness. We show that LEGATO is more computationally efficient than multiple state-of-the-art techniques and more generally robust across a variety of attack settings in practice. We also demonstrate LEGATO's benefits for gradient descent convergence in the absence of an attack.
\end{abstract}

\begin{IEEEkeywords}
federated learning, deep learning, Byzantine attacks, robust aggregation.
\end{IEEEkeywords}

\section{Introduction}
 Federated learning (FL)\cite{mcmahan2016communicationefficient,konen2015federated} is a machine learning paradigm that allows multiple entities, called workers, to train a global model based on the collection of their raw data without exchanging any of it.
 Each worker trains a local model on their own data and periodically sends local gradient updates to a central parameter server. This central server is mostly launched as a cloud-based service~\cite{ibmfl2020, flblog}. The job of the server is to use the received gradients to compute an aggregated gradient, conventionally defined as a simple or weighted average of gradients~\cite{abadi2016tensorflow,mcmahan2016communicationefficient}, which is used to update the global model via usually a gradient descent step.  The workers benefit from receiving the resulting global model because utilizing more data in machine learning typically equates to better models.  
 
 Federated Learning could materialize in an Internet of the things (IoT) setting where edge devices use FL to learn from all collected data without sacrificing limited local storage or the privacy of the collected data.  FL can also be used by a consortium of competing companies, for example, hospitals that want to collaboratively learn a model for diagnosing diseases without violating privacy regulations like the Health Insurance Portability and Accountability Act (HIPAA)\cite{annas2003hipaa} or General Data Protection Regulation (GDPR)\cite{goddard2017eu}. In this setting, each worker itself can use public/private cloud, on-premises, or a hybrid setup to train models locally based on their private data and communicate only model updates in a FL system.
 
If one or more workers within a FL system have inaccurate data or encounter communication errors, the consequently deficient information they provide can compromise the quality of the global model.  This is called a \textit{Byzantine failure}~\cite{10.1145/357172.357176}.
A Byzantine failure may occur unintentionally or may be executed in the form of an attack wherein malicious workers intentionally corrupt a FL system by strategically providing dishonest gradients.  This powerful attack can be performed by a single worker, e.g, Gaussian attacks \cite{xie2018generalized}, and can adversely affect indefinite numbers of models that are used in sensitive areas such as healthcare, banking, and personal devices among others.  Since conventional averaging aggregation schemes are highly vulnerable to Byzantine attacks, it is important to derive alternative methods of aggregating gradients that are Byzantine-robust, meaning the aggregation will mitigate the adverse effects of dishonest gradients in a FL setting.

In this paper, we propose LEGATO, a LayerwisE Gradient AggregaTiOn algorithm that is the first to 
% \textcolor{red}{(remove specifically)} 
use a gradient reweighing scheme at layer level to dampen gradient variance with the intent to increase model robustness in FL. 
LEGATO's performance is consistent across a variety of practical settings and has computational complexity that is linear in terms of the number of workers. These qualities are crucial when deploying FL systems that require large scale participation and involve exposure to unpredictable attack specifications.

\noindent{\bf Contributions.}
Our main contributions are as follows:
\begin{itemize}
    \item Through a pilot study, we show that robustness against Byzantine attacks is layer specific and associated with the variance of the layer-specific $\ell_2$ norms of worker gradients.
    \item We translate the observations from the pilot study into a gradient aggregation algorithm, LEGATO, that reweighs gradients by layer to increase robustness against Byzantine attacks.
    \item We analyze the computational complexity and memory consumption of LEGATO. We prove that LEGATO's computational complexity is linear in terms of the number of workers, which is comparable to conventional, non-robust algorithms, and superior to some robust algorithms including Krum~\cite{blanchard2017byzantinetolerant} and multi-Krum~\cite{blanchard2017byzantinetolerant}.
    \item Through experimentation against a variety of attacks settings and a vanillar FL setting without attacks, including the use of both IID and non-IID local data distributions and different model architectures, we show that LEGATO is consistently more robust than gradient averaging under IID  and overparameterized settings and generally more robust than Krum and coordinate median under non-IID settings.
    %Through experimentation against a variety of attacks settings, including the use of both IID and non-IID local data distributions and different model architectures, we show that LEGATO is consistently more robust than gradient averaging under non-IID settings and generally more robust than Krum. \textcolor{blue}{Update this one}
    % \item We demonstrate LEGATO's general improvements on gradient descent convergence.
\end{itemize}

\noindent{\bf Outline.}
We devoted Section~\ref{sec:background} to discuss the background and motivations of our work.
Section~\ref{sec:study} describes a pilot study of layer-specific gradient response to Byzantine attacks, which informs our proposed algorithm named LEGATO, described in Section~\ref{sec:legato}.  We compare the performance of LEGATO to a baseline algorithm and two state-of-the-art solutions in a variety of settings in Section~\ref{sec:experiments}.
We provide discussion of LEGATO and related work in Section~\ref{sec:remark} and~\ref{sec:related_art} respectively. 
We conclude with final remarks in Section~\ref{sec:conclude}.

\section{Background and Motivations}\label{sec:background}
% \textcolor{blue}{I think we are leading with a lot of details. I suggest adding an introduction on why we care about Non-IID. Then mention why we are studying NN in this settings. I feel explaining the relevance of the problem makes it easier for people to read the solution we are proposing. We should also mention a bit of most robustness stuff is general and does not take advantage of the specific NN setting. }
In this section, we aim to provide more background information about non-IID worker distribution and its associated challenges we may face when dealing with Byzantine threats.
% Specifically, 
% in Section~\ref{}, we will formally introduce the non-iid party distribution and

In a federated learning system, workers will collect and maintain their own training datasets, and no training data samples will leave their owners. 
Workers therefore collect their training data samples from different data sources, for example, in a federated learning task involving collaboration among multiple hospitals, the local training data distribution may be affected by the speciality of the hospital, which might lead to heterogeneous local distribution like some hospitals have more children patients and others have more adult patients.
While the aforementioned case is a feature-level non-IID distribution setting, a label-level non-IID distribution setting is also possible. 
For instance, considering a federated learning system with multiple cellphones as workers trying to train an image recognition model using the photos storing in their cellphones to identify animals. 
For cellphone owners having a dog as a pet, they would have more dog pictures in their local training datasets comparing to cat owners having more cat pictures.

As one of the key features for federated learning, heterogeneous local data distribution has already raised challenges even for vanilla federated learning settings.
In fact, it significantly affects the global model's performance and hence has motivated a new line of research tackling this issue, e.g., \cite{li2018federated,gao2019privacy,wang2020federated} and the references there in.
However, the situation is even more complicated in a real-world scenario since there can be Byzantine threats presented in a federated learning system.

\subsection{Byzantine Threats}
\label{sec:byzantine}

\begin{figure*}[h]
 \centering
  \includegraphics[width=0.75\linewidth]{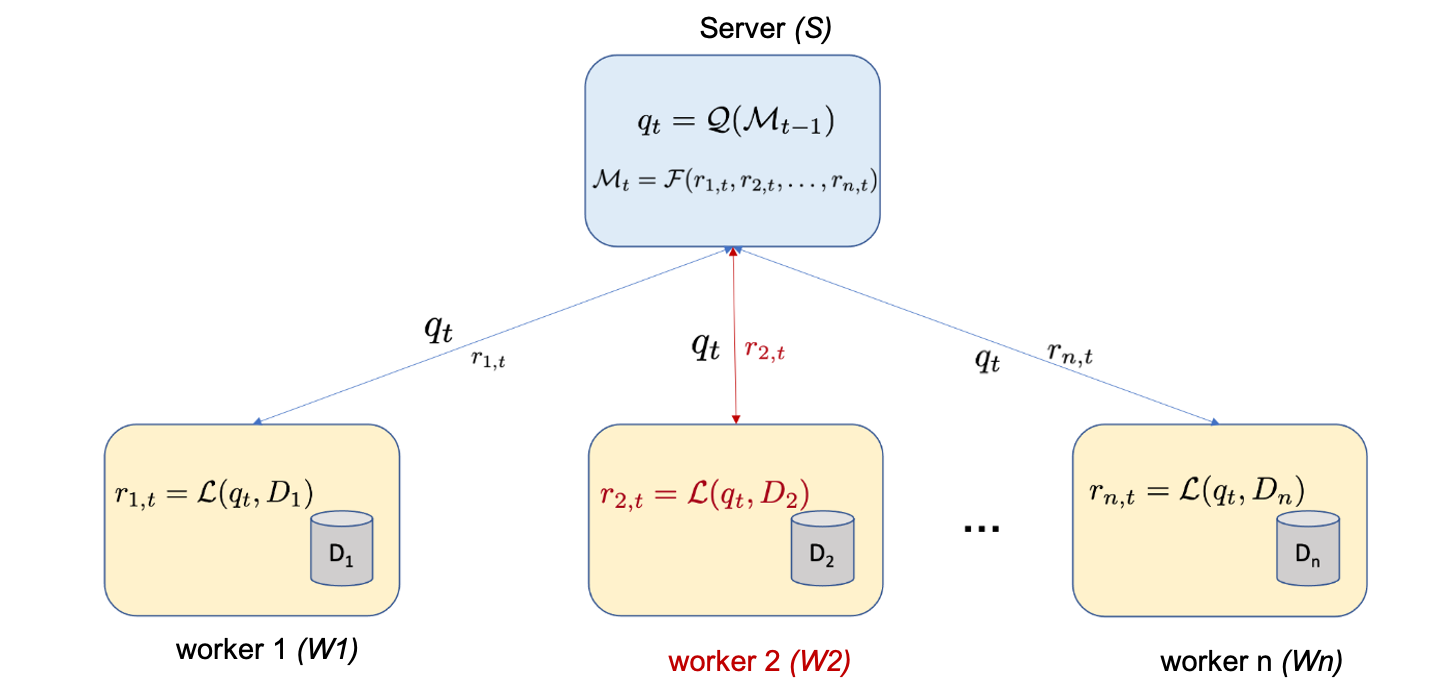}
  \caption{A federated learning system with $n$ workers where the second worker's reply encounters Byzantine failure.}
\label{fig:fl_byzantine}
\end{figure*}

Our goal in this section is to formally define byzantine threats potentially presented in a federated learning system.
In particular, we devote our efforts to discussions on Byzantine failures and Byzantine attacks, which both leads to diminishing final global model performance.

We formally define a \textit{Byzantine Failures}~\cite{10.1145/357172.357176} as in a federated learning system when one or more workers have malfunctions in their computing devices or encountered communication errors, and consequently send deficient information to the parameter server which may compromise the quality of the global model.
As shown in Figure~\ref{fig:fl_byzantine}, consider a federated learning system with $n$ workers where at $t$-round the reply from worker $W_2$ received by the server is corrupted as colored in red. 
% We refer such cases as \textit{Byzantine Failures}. 
%   This is called a \textit{Byzantine failure}~\cite{10.1145/357172.357176}.
Such a Byzantine failure may occur unintentionally or may be executed in the form of an attack, which we refer as \textit{Byzantine Attacks}, wherein malicious entities intentionally corrupt a FL system by providing strategically dishonest replies. This powerful attack can be performed by only a single worker like $W_2$ in Figure~\ref{fig:fl_byzantine} and can adversely affect indefinite numbers of models that are used in sensitive areas such as healthcare, banking, and personal devices and so on. 
We refer the worker unintentionally sending inaccurate replies or intentionally performing \textit{Byzantine attacks} as a \textit{Byzantine worker}.

\paragraph{\bf  Gaussian Attack \cite{xie2018generalized}} This type of Byzantine attack can be performed by a single worker in a federated learning system, and do not need any collaboration among Byzantine workers. Workers performing this attack will randomly sample their replies from a Gaussian distribution, $N(\mu, \sigma^2I)$ regardless their local training datasets.

\paragraph{\bf Fall of Empires Attack \cite{xie2019fall}}
Designed to break robust aggregation algorithms like, Krum \cite{blanchard2017byzantinetolerant} and coordinate median \cite{yin2018byzantinerobust} , it requires a minimum number of Byzantine workers, depending on the robust algorithms, to collaboratively construct the attack.
Moreover, it assumes the Byzantine workers know about the replies sent from the honest workers.
Let $v_i, i=1, ...,h$ to be the replies sent by $h$ honest workers, the malicious replies sent by the Byzantine workers who perform the Fall of Empires attack will send replies, $u_j$, formulated as:
\begin{equation}
u_1 = u_2 = ... = u_{n-h} = -\frac{\epsilon}{h}\sum_{i=1}^h v_i
,
\end{equation}
where $\epsilon$ is a real, non-negative factor that depends on the lower and upper bounds of the distances between each reply.
% Detail definition of this attack can be found in Chapter~\ref{chpX}.

\subsection{Challenges}
The aforementioned several Byzantine attacks are very powerful. We can see from Fig~\ref{fig:attvsnat} that Gaussian attack with $\sigma$ having larger values creates more damage to the global model's performance.

\begin{figure}[h!]
\centering
  \includegraphics[width=\linewidth]{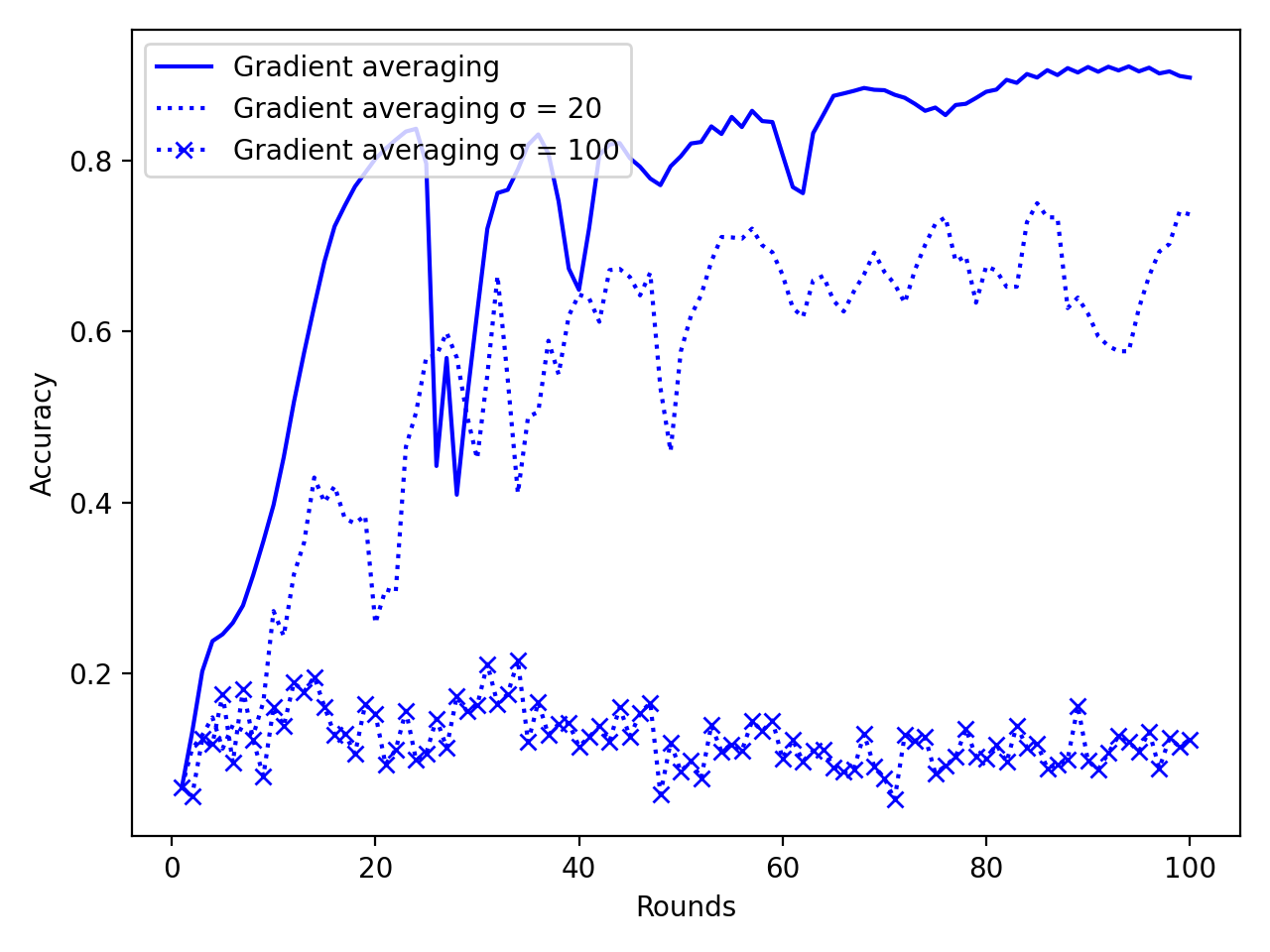}
  \caption{This experiment involves $25$ workers with $1,000$ IID data points from MNIST. We compare the performance of gradient averaging where $4$ Byzantine workers executing a Gaussian attack with $\mu = 0$ and $\sigma$. Here we use batch size of $50$ and learning rate of $0.05$.}
\label{fig:attvsnat}
\end{figure}

Although Krum and coordinate median are proved to successfully identify the malicious Byzantine workers in the setting where workers have independent identical distributed (IID) local datasets, see Figure~\ref{fig:iidvsnoniid}, they failed significantly when workers have heterogeneous local dataset.
It can also be shown that even without any Byzantine attacks presented these robust aggregation algorithms cannot guarantee a global model with good performance.
\begin{figure}[h!]
\centering
  \includegraphics[width=\linewidth]{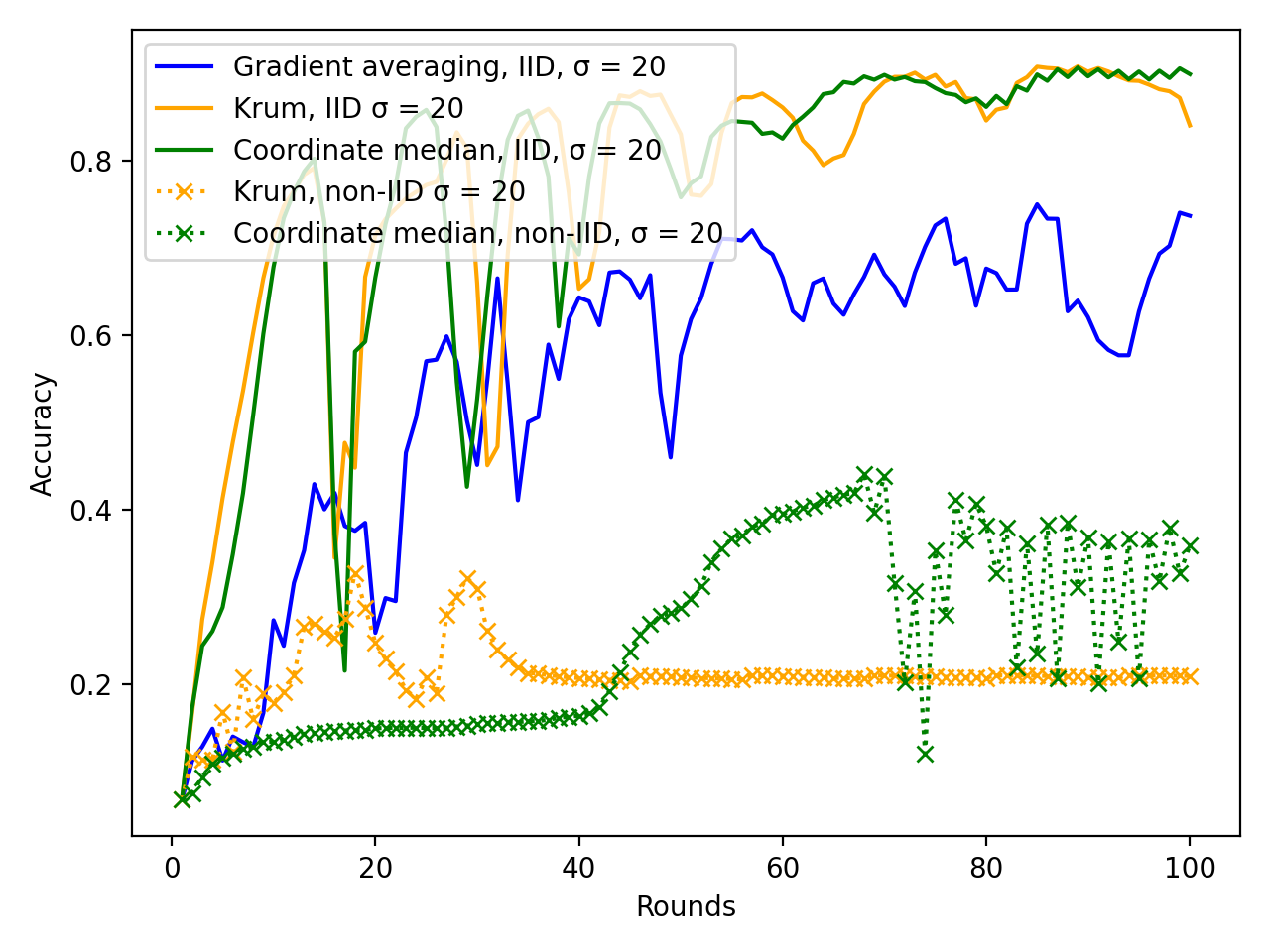}
  \caption{This experiment involves $25$ workers with $1,000$ data points randomly sampled from MNIST for IID setting and $1,000$ points containing only $1$ label per worker for non-IID setting, respectively. We compare the performance of gradient averaging, Krum and Coodinate median where $4$ Byzantine workers executing a Gaussian attack with $\mu = 0$ and $\sigma=20$. Here we use batch size of $50$ and learning rate of $0.05$.}
\label{fig:iidvsnoniid}
\end{figure}

So far we have shown empirically that several well-known robust aggregation algorithms cannot defend against Byzantine attacks especially in the non-IID setting. Next, we discuss some challenges involved in this setting.

\noindent{\bf Challenge One:} Can a robust aggregation algorithm smartly detect whether a worker sending ``different'' reply is a malicious worker or a benign worker with different local data distribution?

\noindent{\bf Challenge Two:} Can a robust aggregation algorithm train a global model with reasonable performance under non-IID local data distribution setting?

\noindent{\bf Challenge Three:} Can a robust aggregation algorithm utilizes all the information collected from workers to diagnosis the workers' behavior?

It is well-known that different layers of a neural network have different functionality, and may behave different towards different input data samples which could be used for mitigating the effect of attack. 
In the next section, we focus on investing how each layer of the neural network reacts to Byzantine threats in a federated learning system.
% Next, we discuss gradient averaging, Krum and Coordinate median from a theoretical perspective and point out the drawbacks of these algorithms in our next Section.

% Since conventional averaging aggregation schemes are highly vulnerable to this attack, it is important to derive alternative methods of aggregating gradients that are Byzantine-robust, meaning the aggregation will mitigate the adverse effects of dishonest gradients.

\section{A Pilot Study of Layer Robustness}\label{sec:study}

Many existing robust aggregation algorithms utilize some form of outlier detection to filter out Byzantine gradients based on the assumption that, in comparison to Byzantine gradients, honest gradients will be ``similar''.  
Theoretically, the relative similarity of gradients will decrease as the similarity of their local distributions decreases.
In this section, we observe the effect of Byzantine threats from a different aspect, i.e., observing how different layers of the neural network react to a Byzantine attack in the FL system. 
In Figure~\ref{fig:pilot}, we highlight some results that illustrate key observations from our study of the response of gradients per layer to Byzantine attacks.
In this study, we train a simple convolutional neural network (CNN) with two convolutional layers and dense layers, 
% which we refer as conv1, conv2, dense 1 and dense 2,
whose trainable layer weights and biases are listed in Table \ref{table:1}, 
in a FL system with ten workers. Each worker has $1,000$ data points randomly drawn from the MNIST dataset ~\cite{Lecun98gradient-basedlearning} so their local distributions are IID.  We therefore expect local gradients to be relatively similar. By using $\ell_2$ norms to measure the difference among them, we confirm this expectation and further study the similarity at a layer level and with the inclusion of Byzantine gradients. But first, we provide the notation used throughout this paper.

\begin{figure}[h!]
\centering
  \includegraphics[width=\columnwidth]{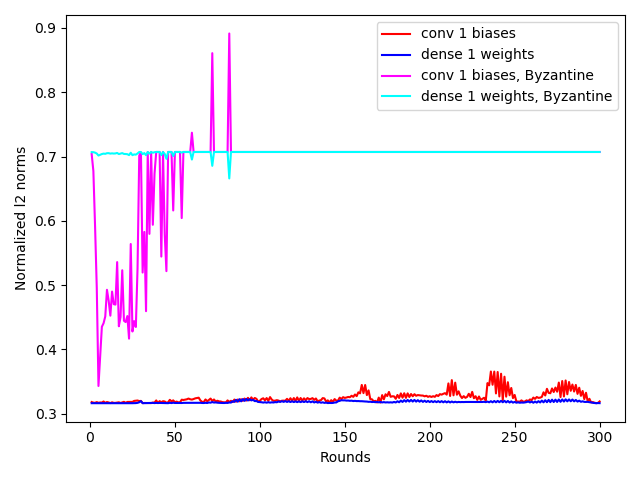}
  \caption{\footnotesize Normalized $\ell_2$ norms associated with the biases of the first convolutional layer and the weights of the first dense layer.  This experiment involves 10 workers with $1,000$ IID data points from MNIST.  We compare results with and without 2 Byzantine workers executing a Gaussian attack with $\mu = 0$ and $\sigma = 200$.}
\label{fig:pilot}
\end{figure}

\noindent{\bf Notation.}
We use $n$ to represent the total number of workers in a FL system.  To denote the $d$-dimensional gradients of worker $p$ for layer $l$ at global training round $t$, we use $g^t_{p,l}\in \mathbb{R}^d$. 
To denote the complete set of gradients of worker $p$ at round $t$, we use $G_p^{t}$.
Finally, we use $\mathcal{G}^t$ to represent a list of collected worker gradient vectors at round $t$.

\begin{table}[h!]
\centering
\begin{tabular}{|c c c c|} 
 \hline
 Order & Layer & Parameters & Dimension\\ [0.5ex] 
 \hline
 \multirow{2}{*}{1} & \multirow{2}{*}{conv 1} & weights & 288\\
 & & biases & 32\\
 \multirow{2}{*}{2} & \multirow{2}{*}{conv 2} & weights & 18432\\
 & & biases & 64\\
 \multirow{2}{*}{3} & \multirow{2}{*}{dense 1} & weights & 1179648\\
 & & biases & 128\\
 \multirow{2}{*}{4} & \multirow{2}{*}{dense 2} & weights & 1280\\
 & & biases & 10\\ [1ex] 
 \hline
\end{tabular}
\caption{List of trainable layers of a simple CNN.}
\label{table:1}
\end{table}

At round $t$, we compute the normalized $\ell_2$ norm of a specific layer, $P^t_l$ across all parties as follows: 
\begin{equation}\label{eq:normalized-l2}
  P^t_{l} \leftarrow \frac{\parallel [g^t_{1,l},g^t_{2,l},...,g^t_{n,l}]\parallel_2}{\sum_{p=1}^n x_p},
 \end{equation} 
where $x_p$ is the norm of the current gradient of worker $p$, i.e., $x_p = \parallel <g^t_{p,0}, g^t_{p,1}, \dots, g^t_{p,l}>\parallel_2$, and $[g^t_{1,l},g^t_{2,l},...,g^t_{n,l}]$ denotes a matrix with the $p$-th row being the $l$-th layer of worker $p$'s current gradient.
This is essentially measuring the variation across workers of the variation of their gradients, and the normalization (division by the $\ell_1$ norms of current gradients) eliminates the possible interference of the magnitude of gradient vectors on the comparison of gradients.  We compute these normalized $\ell_2$ norms across all ten workers at each round and compare the results between the layer weights and biases over time.  We run this experiment both with and without two out of ten workers being Byzantine.  These Byzantine workers execute a Gaussian attack by sending gradients that are randomly sampled from a Gaussian distribution, $N(\mu, I\sigma^2)$. We use the setup from~\cite{xie2018generalized} with the distribution $N(0, 200^2I)$.

Figure~\ref{fig:pilot} compares the results for two layer parameters' gradients under non-attack and Byzantine attack scenarios.  In absence of the Byzantine attack, the bias gradients from the first convolutional layer have similar $\ell_2$ norms to the weight gradients from the first dense layer until around round 150.  Then, the convolutional bias gradient $\ell_2$ norms become clearly larger and more varied across rounds than the weights of the first dense layer.  In the Byzantine case, the convolutional bias norms start to even more decisively exceed the dense weight norms even earlier in training.  While the variance of the dense $\ell_2$ norms across rounds is comparatively small in both settings, the variance in the convolutional layer norms is significantly amplified with the addition of Byzantine workers.  These patterns demonstrate the fact that the gradient variance imposed by the Byzantine workers is more drastically affecting the convolutional bias gradients than the dense weights.  They also have larger and more varied norms across rounds in the attack and non-attack settings separately.  Therefore, we conclude more generally that the layers whose gradient norms vary more across rounds have greater inherent vulnerability that is more intensely exploited by Byzantine attacks.

\section{Layerwise Gradient Aggregation}\label{sec:legato}
This section describes LayerwisE Gradient AggregatTiOn (LEGATO): an aggregation algorithm that utilizes the observations from the pilot study to reweigh gradients based on layer-specific robustness.  

\subsection{Threat Model}
LEGATO's goal is to be able to overcome erroneous or malicious gradients and it works under the following threat model. We assume that the server is honest and wants to detect malicious or erroneous gradients received during the training process.
Workers may be dishonest and try to provide gradient updates designed to evade detection. To ensure LEGATO is resilient against existing strenuous attacks such as \cite{xie2019fall}, we assume that malicious workers may collude and obtain access to gradients of other workers to perform those attacks. Under this threat model, we design and assess LEGATO. We now present our LEGATO algorithm.

\subsection{LEGATO}
We frame LEGATO, presented in Algorithm~\ref{alg:LEGATO}, in a vanilla FL pipeline (see Algorithm~\ref{alg:FL}). 
% We make use of the following notation.
%\textcolor{green}{I added this description here help}
%We use the following notation. To denote the gradient vector of worker $p$ for layer $l$ we use $g_{p,l}$. 
%To denote the complete set of gradients of worker $p$ at round $t$, we use $G_p^{t}$.
%Finally, we use $\mathcal{G}^t$ to represent a list of all worker gradients vectors at round $t$.
%Finally, we use $\mathcal{G}^t$ := one entry in the log which is the aggregated value for a period.
% 
%\textcolor{green}{$\mathcal{G}^t$ := one entry in the log which is the aggregated value for a period \\ 
%$G_p^{t}$:= Gradient of worker $p$, optional at round $t$\\
%in some cases, we don't \\ 
%$g_{p,l}$:= gradient of worker $p$ for layer $l$.}
% 
For each round of training, the novel steps of LEGATO start when the server receives gradients from all of the workers.  We solely consider a synchronous setting, meaning the server will wait until it receives responses from all queried workers.  The server keeps a log, with maximum size $m$, of the most recent past gradients from all workers.  At round $t$, the log is denoted by $\textit{GLog}:=[\mathcal{G}^{t-m}, \mathcal{G}^{t-m+1}, \dots, \mathcal{G}^{t-1}]$.

\begin{algorithm}[!t]
\caption{\textsc{Federated Learning with LEGATO.} Server trains the global model, $\mathcal{M}_{\text{global}}$, for a \textrm{MaxRound} rounds.  Worker $p\in \mathcal{P}$ owns its local dataset, $\mathcal{D}_{p}$.}
\label{alg:FL}
\SetAlgoLined
\textbf{Server executes:}

\Indp
\For{each training round $t$ in \textrm{MaxRound}}{
$\mathcal{G}^t \leftarrow$ new list\;
\For{each worker $p\in \mathcal{P}$}{
$G_{p} \leftarrow \textbf{GradientRequest}(\mathcal{M}_{\text{global}})$\;
add $G_{p}$ to $\mathcal{G}^t$\;}
\tcp{\footnotesize Aggregates gradients}
% (pseudocode in Algorithm~\ref{alg:LEGATO})}
$\mathcal{G}^t_{agg} = \text{LEGATO}(\mathcal{G}^t)$\ \tcp{\footnotesize(Algorithm 2)}

Update $\mathcal{M}_{\text{global}}$ via gradient descent using $\mathcal{G}^t_{agg}$\;}

\Indm

\BlankLine

\textbf{Worker executes:}

\Indp
\textbf{GradientRequest}($\mathcal{M}_{\text{global}}$):

\Indp
$G_p \leftarrow \text{\bf GetGradient}(\mathcal{M}_{\text{global}},\ \mathcal{D}_{p})$\;
\Return{$G_p$}

\Indm

\Indm
\end{algorithm}

\begin{algorithm}[!t]
\caption{\textsc{LEGATO.} Algorithm to aggregate gradients at round $t$, given a list of current worker gradients, $\mathcal{G}^t$.  The server maintains a log of recent past worker gradients with maximum size $m$: $\textit{GLog}:=[\mathcal{G}^{t-m}, \mathcal{G}^{t-m+1}, \dots, \mathcal{G}^{t-1}]$, where $\mathcal{G}^t=[G^t_1, G^t_2, \dots, G^t_{|\mathcal{P}|}]$.}
\label{alg:LEGATO}
\SetAlgoLined
\textbf{Server executes:}

\Indp
\eIf{$t=1$}{
 Initialize an empty log \textit{GLog}\;}{
\textbf{UpdateGradientLog($GLog$, $\mathcal{G}^t$)}\;}
\For{$\mathcal{G}^i$ in \textit{GLog}\label{alg:ln:norm1}}{
\For{$p$ in $|\mathcal{P}|$}{$x_p = \parallel G^{i}_{p}\parallel_2$\;}
 \For{each layer $l$}{$P^i_{l} \leftarrow \frac{\parallel [g^{i}_{1,l},g^i_{2,l},...,g^i_{n,l}]\parallel_2}{\sum_{p\in \mathcal{P}}x_p}$\;}}\label{alg:ln:norm2}
\For{each layer l\label{alg:ln:factor1}}{$w_l \leftarrow Normalize(\frac{1}{\sqrt{Var(P^1_{l},...,P^t_{l})}})$\;}\label{alg:ln:factor2}
\For{$p$ in $|\mathcal{P}|$ \label{alg:ln:weight1}}{
 \For{each layer l}{$G^*_{p,l} \leftarrow {w_l}G^t_{p,l} + \frac{1 - {w_l}}{m-1}\sum_{j = 1}^{m-1} G^{t-j}_{p,l}$\;}}\label{alg:ln:weight2}
\Return{{$\mathcal{G}^t_{agg} \leftarrow \sum_{p\in \mathcal{P}} G^*_{p,l}$\;}}\label{alg:ln:avg}
\BlankLine
\textbf{UpdateGradientLog($\textit{GLog}$, $\mathcal{G}^t$):}

\Indp
 $\textit{GLog} \leftarrow \textit{GLog} + \mathcal{G}^t$\;
 \If{$len(\textit{GLog}) > m$}{$\textit{GLog} \leftarrow \textit{GLog}[1:]$}
\end{algorithm}
%  \vspace{-.1in}
First, the server updates the gradient log, \textit{GLog}, to contain the most recent $m$ gradients collected from workers. 
In Lines~\ref{alg:ln:norm1}-\ref{alg:ln:norm2}, the server uses \textit{GLog} to compute the layer-specific normalized $\ell_2$ norms $P^i_l$ in the same way as \eqref{eq:normalized-l2}.  It then uses the reciprocal of the standard deviation of these norms across all logged rounds as a robustness factor that is assigned to each layer and normalized across all layers (Lines~\ref{alg:ln:factor1}-\ref{alg:ln:factor2}).  It is applying the observation from the pilot study that the less the $\ell_2$ norms vary across rounds, the more robust a layer is.  In Lines~\ref{alg:ln:weight1}-\ref{alg:ln:weight2}, each worker's gradients are reconfigured as a weighted sum of the average of the worker's logged gradients and its current gradient vector.  The weights are chosen as a function of the robustness factor per layer that allows the updates of less robust layers to relies more heavily on the average of past gradients.  Finally, in Line~\ref{alg:ln:avg}, all of these reweighed gradients are averaged across all workers and the result is used as round $t$'s aggregated gradient, $\mathcal{G}^t_{agg}$.  This reweighing strategy is ultimately dampening the noise across gradients that may be a result of a Byzantine attack, with the goal of reversing the Byzantine gradients' effect of pulling the aggregated gradient away from an honest value.  This allows LEGATO to mitigate the attack without needing to identify and filter out attackers.  The fact that we focus this dampening at the most vulnerable layers allows reliable layers' current gradients, which are most accurate, to still weigh heavily into the aggregated gradient, thus limiting the sacrifice of convergence speed that could result from using past gradients.
Furthermore, the online robustness factor computation allows LEGATO to generalize to a variety of model architectures because it adopts online factor assignments rather than an absolute quantification of robustness.  As is evidenced in~\cite{zhang2019layers}, the knowledge of layer-specific robustness varies between architectures so an online and model-agnostic approach is ideal.

%\subsubsection{Tuning the Specifications.}
%LEGATO's novelty comes from its dynamic computation of layer-specific robustness factors for reweighing gradients.  In our experiments, we use a specified log size, robustness measure, translation of the measure into a robustness factor, and single frequency of the measurement and translation.  However, these specifications can easily be modified to 
%better comply with other knowledge or experiments.

\subsection{Complexity analysis of LEGATO}
The most efficient gradient aggregation algorithms have computational complexity that is linear in terms of the number of workers $n$ in a FL system.  Among them are conventional, non-robust algorithms like gradient averaging and Byzantine-tolerant algorithms that utilize robust statistics, for example, coordinate median.  Proposition 1 states that LEGATO has time complexity $\mathcal{O}(dn+d)$, meaning it is also linear in terms of $n$.  This is a crucial improvement that LEGATO has over state-of-the-art robust algorithms such as Krum and Bulyan~\cite{mhamdi2018hidden} whose time complexities are quadratic in $n$.  LEGATO's maintenance of a log introduces a space requirement of $\mathcal{O}(dnm)$, which is formalized in Proposition 2.
\begin{proposition}
LEGATO has time complexity $\mathcal{O}(dn + d)$.
\end{proposition}
\begin{proof}
First, for each logged round at each layer, the server computes the $\ell_2$ norms of each worker's gradients($\mathcal{O}(dn)$).  Then, for each logged round at each layer, the server computes the $\ell_2$ norm of the matrix of worker gradients and normalizes by dividing by the $\ell_1$ norm of all workers' $\ell_2$ gradient norms ($\mathcal{O}(dn)$).  Then it computes the standard deviation across these normalized round $\ell_2$ norms for each layer ($\mathcal{O}(dn)$).  Next, the server normalizes the reciprocals of the standard deviations corresponding to each layer ($\mathcal{O}(d)$).  It then computes the average of logged gradients from each worker at each layer ($\mathcal{O}(dn)$).  Lastly, the gradients are reweighed by worker and averaged across all workers ($\mathcal{O}(dn)$).

Note that iterating through each individual layer of gradients does not add a factor of $l$ to the time complexity because $d$ encompasses all gradients at all layers.  Therefore, iterating through all gradients by layer is $\mathcal{O}(d)$ and iterating through all gradients in a flattened vector is also $\mathcal{O}(d)$.
\end{proof}

\begin{proposition}
LEGATO has space complexity $\mathcal{O}(dnm)$
\begin{proof}
The log stores $m$ past gradients from each worker.
\end{proof}
\end{proposition}

\section{Evaluation}\label{sec:experiments}

In a variety of attack settings and non-attack settings, we evaluated
LEGATO's performance against those of \textit{gradient averaging}~\cite{abadi2016tensorflow}, which is a simple average of all local gradients that serves as a baseline, and
two state-of-the-art robust aggregation algorithms:
\textit{Krum}\cite{blanchard2017byzantinetolerant} and \textit{coordinate median}\cite{yin2018byzantinerobust}.
%In all our experiments all workers have the same amount of local data, hence gradient averaging has the same effect as FedAvg~\cite{mcmahan2016communicationefficient}, a popular aggregation algorithm that uses a weighted average wherein worker updates are weighted by their relative amount of local data samples.  

We empirically support the following claims:
\begin{itemize}
    \item In settings without strictly bounded honest gradient variance, which we demonstrate through experiments with and without IID data, LEGATO is more robust than Krum and coordinate median.
    % \item In all settings we test, LEGATO is more robust than gradient averaging~\cite{abadi2016tensorflow}.
    \item Considering all attack settings, LEGATO is the most generally robust of the three algorithms.
    \item LEGATO has the best performance of the three algorithms in the absence of a Byzantine attack.
    % when LEGATO is better than avg?
    % \item LEGATO can 
\end{itemize}

\begin{figure}[h!]
% \centering
% \begin{minipage}{\columnwidth}
\begin{subfigure}{\columnwidth}
  \centering
  \includegraphics[width=\columnwidth]{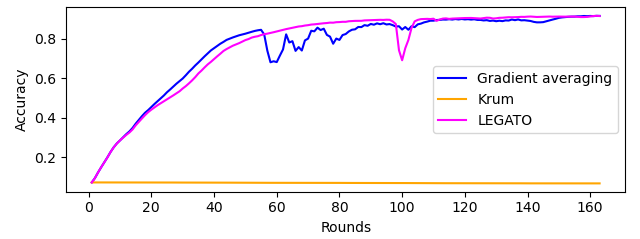}
  \caption{{\footnotesize MNIST dataset, each worker has 1k IID training images, the batch size is $50$, and the learning rate is $.05$.}}
%   \vspace{-.1in}
  \end{subfigure}
%   \noindent{\footnotesize (a) MNIST dataset, each worker has 1k IID training images, the batch size is $50$, and the learning rate is $.05$.}
%   \label{fig:noniid_nonattack}
%  \end{minipage}%
%   \begin{minipage}{\columnwidth}
\begin{subfigure}{\columnwidth}
  \centering
  \includegraphics[width=\columnwidth]{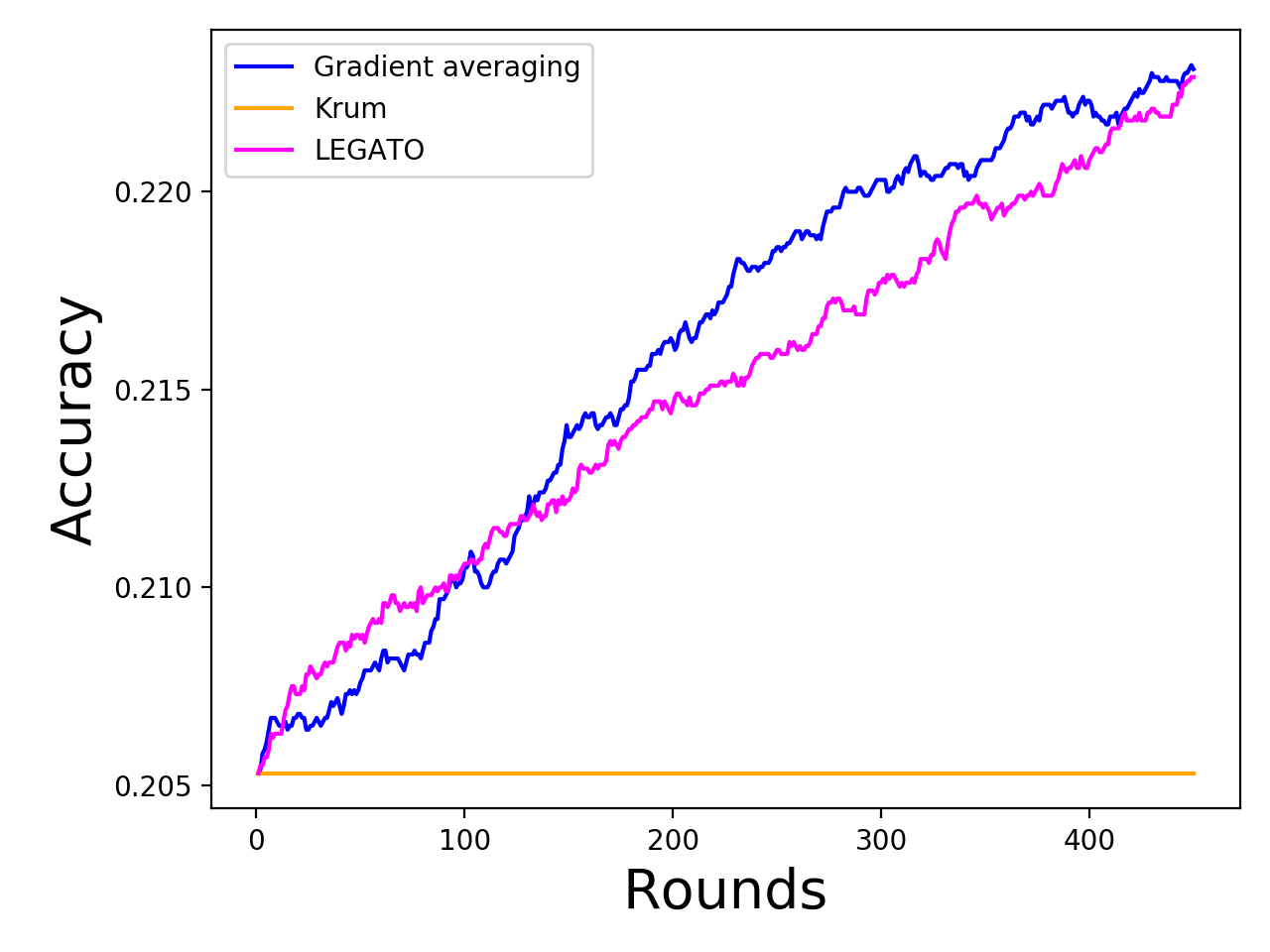}
  \caption{{\footnotesize CIFAR10 dataset, each worker has 2K IID training images, the batch size is 1k, and the learning rate is $.01$.}}
%   \vspace{-.1in}
\end{subfigure}
%   \label{fig:noniid_nonattack}
%  \end{minipage}%
  \caption{\footnotesize 
  Accuracy for MNIST (\textit{top}) and CIFAR10 (\textit{bottom}) under \textit{Fall of Empires attack}, where $11$ out of $25$ workers execute the attack with $\epsilon = .001$.  
%   For MNIST, 
%   For CIFAR10, 
%   each worker has 2K IID training images, the batch size is 1k, and the learning rate is $.01$.
  }
%   , and the log size is 10.}
% \vspace{-.1in}
\label{fig:foe}
\end{figure}

\subsection{Experimental Setup}
We implemented LEGATO using the IBM Federated Learning library \cite{ibmfl2020}.
All experiments in this section train a global model to classify MNIST handwritten digits or to classify CIFAR10 images in a FL system with $25$ workers. Each worker has training data points that are randomly sampled either across the entire dataset (IID) or across only one class each (non-IID setup taken from ~\cite{zhao2018federated}) which will be specified later. 
The workers compute the gradient based on its local dataset before responding to the server's request at every global training round.  
The global model for MNIST is a simple CNN with two convolutional layers, one dropout layer, and two dense layers.  For CIFAR10, it is a CNN with three convolutional layers followed by three dense layers. We use Adadelta optimization and past gradient log size $10$.  We use accuracy across global training rounds as a performance metric, which is comparable to F1 score in these experiments.

We assess LEGATO's performance against two types of Byzantine attacks: the Fall of Empires attack~\cite{xie2019fall} and the Gaussian attack.
\subsubsection{Fall of Empires Attack}
The Fall of Empires attack~\cite{xie2019fall} is specifically designed to break Krum and coordinate median. Byzantine workers require knowledge of the $h$ honest gradients, $v_i$, and send gradients, $u_j$, formulated as:\\
% \begin{equation}
$u_1 = u_2 = ... = u_{n-h} = -\frac{\epsilon}{h}\sum_{i=1}^h v_i
$,
% \end{equation}
where $\epsilon$ is a real, non-negative factor.
\subsubsection{Gaussian Attack}
For the Gaussian attack, Byzantine workers randomly sample gradients from a Gaussian distribution, $N(\mu, \sigma^2I)$.  In our experiments, we set $\mu=0$, and vary the selection of $\sigma$.
% , to be either relatively large or relatively small.

\begin{figure}
%   \centering
%   \begin{minipage}{0.49\columnwidth}
%   \centering
  \begin{subfigure}{\columnwidth}
  \centering
  \includegraphics[width=\columnwidth]{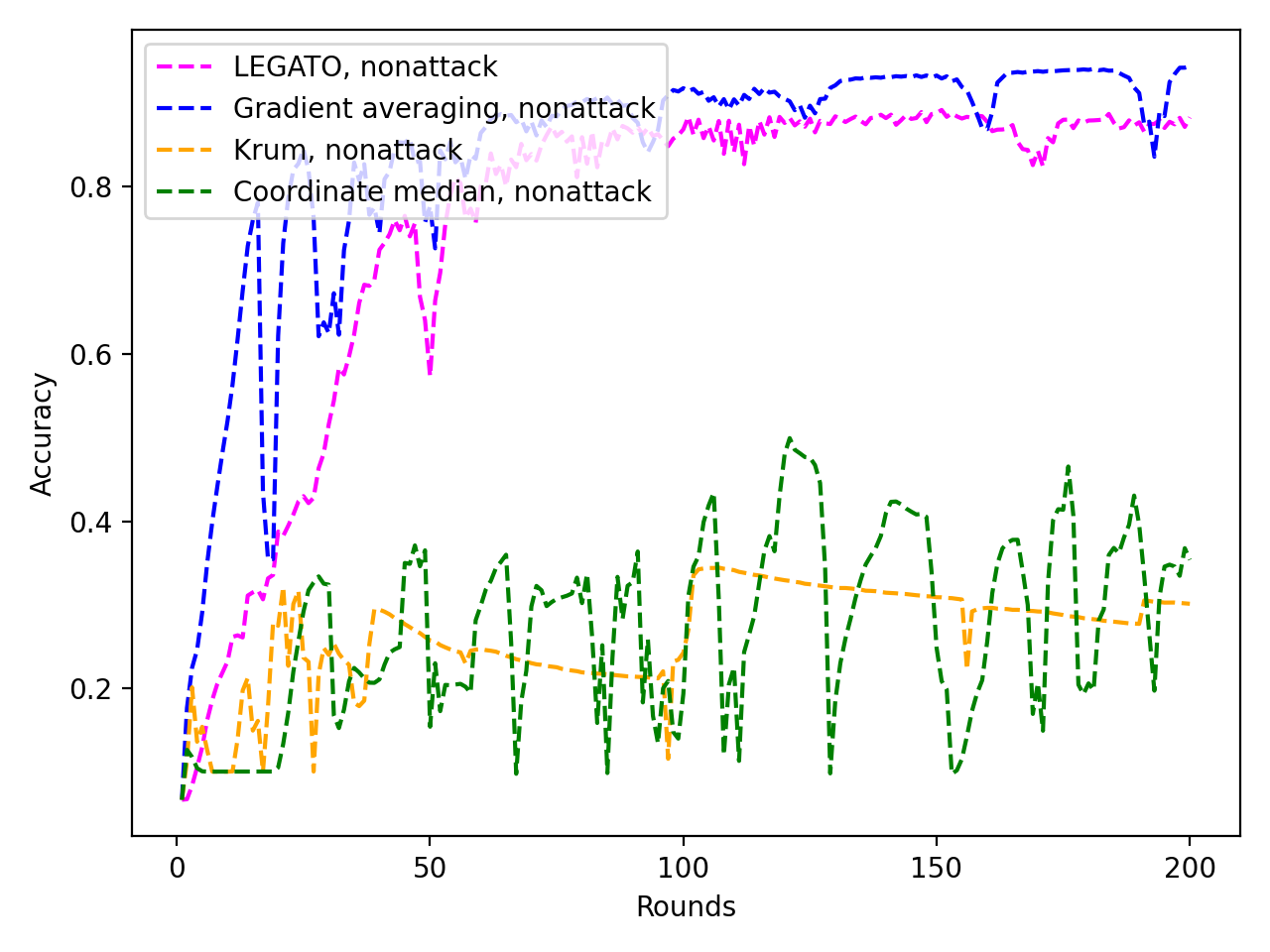}
  \caption{Non-attack setting.}
  \end{subfigure}
  \begin{subfigure}{\columnwidth}
  \centering
  \includegraphics[width=\columnwidth]{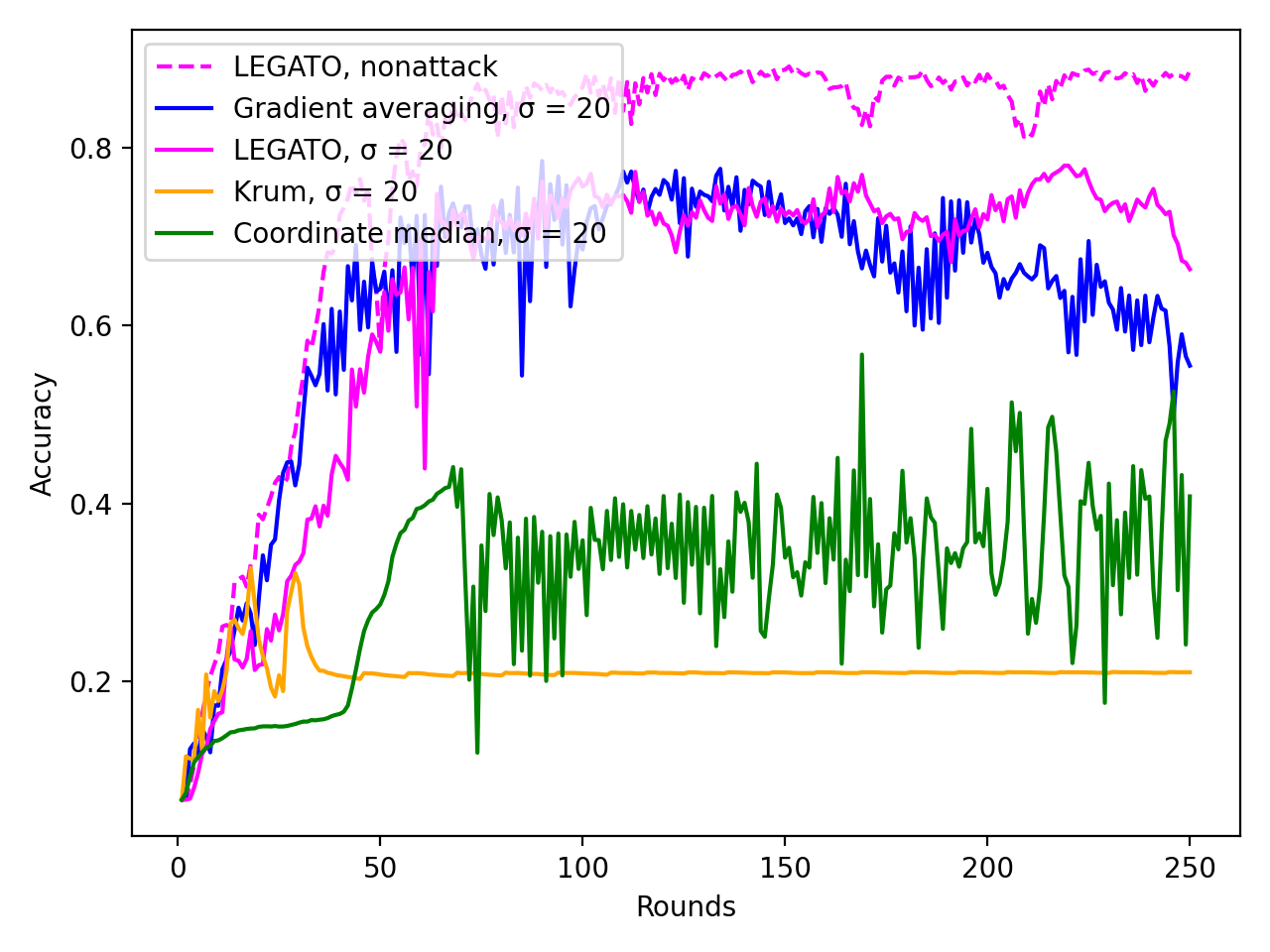}
  \caption{Gaussian attack setting with $\sigma=20$.}
  \end{subfigure}
  \begin{subfigure}{\columnwidth}
  \centering
  \includegraphics[width=\columnwidth]{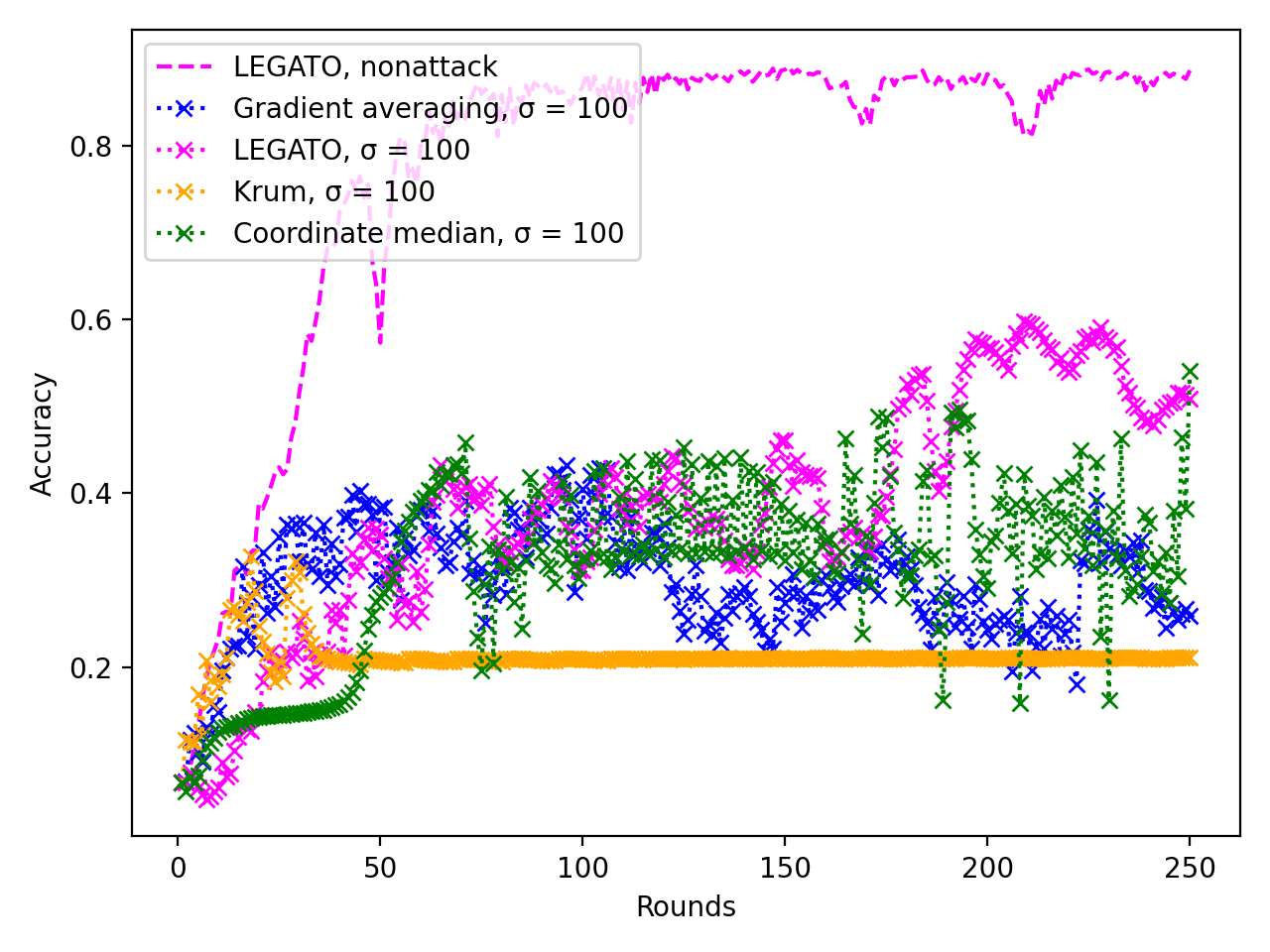}
  \caption{Gaussian attack setting with $\sigma=100$.}
  \end{subfigure}
  
 \caption{Accuracy for \textit{non-attack setting} (\textit{top}) and under Gaussian attack (\textit{bottom}), where $4$ out of the $25$ workers are Byzantine and randomly sample gradients from $N(0, \sigma^2I)$. 
 For all cases, the workers have non-IID data distribution, the batch size is $50$ and the learning rate is $.03$.}
%   \vspace{-.1in}
\label{fig:noniid}
\end{figure}

\subsection{Comparison of LEGATO against Krum and coordinate median}\label{sec:krum}

%In this section, we highlight results for experiments that exploit a weakness, i.e., the assumption of bounded honest gradient variance, of Krum and coordinate median, resulting in the algorithm's poor performance.  
%LEGATO does not require this assumption and hence demonstrates better performance.
% \vspace{-.1in}

% \vspace{-.1in}

We first study the case when there is an attack presented. Figure~\ref{fig:foe} shows results for algorithms against the Fall of Empires attack for training on MNIST (top) and CIFAR10 (bottom). As it can be seen in the figure, the Fall of Empires attack is effective against Krum, but not against LEGATO or gradient averaging.
With this attack, all Byzantine gradients are the same and hence form a large cluster.  Krum identifies these Byzantine worker gradients as being more trustworthy and will use them to update the global model.
% because they are close in distance to many other workers' gradients. 
This takes advantage of Krum's reliance on the assumption that honest gradients are sufficiently clumped ~\cite{elmhamdi2020distributed} and similar results are shown on coordinate median.

This assumption of bounded honest gradient variance is especially unrealistic for FL systems with non-IID worker local data distributions.  
Therefore, when local workers' data distribution is non-IID, Figure~\ref{fig:noniid} shows that Krum and coordinate median performs poorly regardless if there is a Gaussian attack presented or not.
In fact, they cannot distinguish between Byzantine gradients and the benign gradients from workers owning different local distributions. 
Our experimental results demonstrate that LEGATO is independent of this assumption.
In particular, LEGATO achieves consistently better accuracy comparing to gradient averaging for all attack settings. 
When the Gaussian attack uses a small variance, $\sigma=20$, LEGATO achieves a significant improvement in model accuracy compared to Krum and coordinate median.  
When $\sigma=100$, the accuracy of LEGATO is better than Krum and slightly exceeds coordinate median after $175$ rounds because the high standard deviation exploits a vulnerability LEGATO has to extreme outliers.

\subsection{Comparison of LEGATO Against Robustness of Gradient Averaging}\label{sec:grad_avg}
% \vspace{-.1in}
% \begin{figure}[h!]
% \centering
%   \includegraphics[width=\columnwidth]{figures/new_g20_mnist_m.png}
% %   \vspace{-.3in}
%   \caption{\footnotesize Accuracy under Gaussian attack setting where $4$ of $25$ workers are Byzantine and randomly sample gradients from $N(0, 20^2I)$. 
%   All workers' data distribution is IID and uses minibatch size $50$ with learning rate is $.05$. }
% \label{fig:iid_gaussian}
% % \vspace{-.1in}
% \end{figure}
% \vspace{-.3in}
\begin{figure}[h!]
  \centering
  \includegraphics[width=\columnwidth]{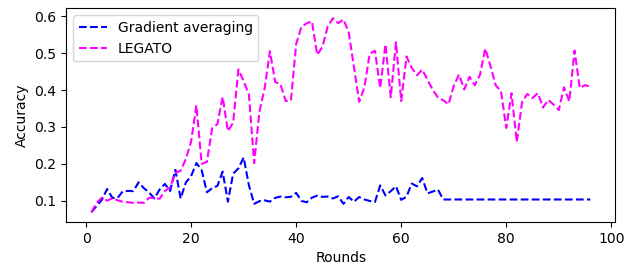}
%   \vspace{-.1in}
  \caption{\footnotesize Comparing the accuracy of gradient averaging and LEGATO for training a CNN with 8 convolutional model layers instead of the previous architecture with 2 convolutional layers.  The minibatch size is $50$ and the learning rate is $.05$, and $4$ of $25$ workers are Byzantine workers executing a Gaussian attack where gradients are randomly sampled from $N(0, 20^2I)$.}
\label{fig:overparameterized}
% \vspace{-.1in}
\end{figure}

Figure~\ref{fig:overparameterized} shows that the performances of LEGATO and gradient averaging drop and the performance gap widens when more convolutional layers are added to the model architecture.  This is increasing the overparamterization of the model, which is a condition worth studying due to its common existence in practice.  As the number of parameters increases, the model becomes more susceptible to learning the Gaussian noise in the same way that it would become more susceptible to learning noise in training data and hence overfitting in non-attack settings.  LEGATO's dampening of gradient oscillations is therefore resembling the effect of regularization methods for reducing overfitting.

\subsection{Analysis of Layer Robustness Factor Results}\label{sec:factors}

We also compare how different neural network layers behave when the training process is being attack versus in a non-attack setting.
In this experiment, we use IID worker data distribution; in the attack scenario $4$ out of $25$ workers launched Gaussian attack with $\sigma=20$. 
%In this subsection, we analyze the layer robustness factor by comparing the robustness factors obtained from FL systems with IID party data distribution where $4$ out of $25$ parties launched Gaussian attack with $\sigma=20$ and where there is no attack. 

\begin{figure}[h!]
  \centering
  \includegraphics[width=\columnwidth]{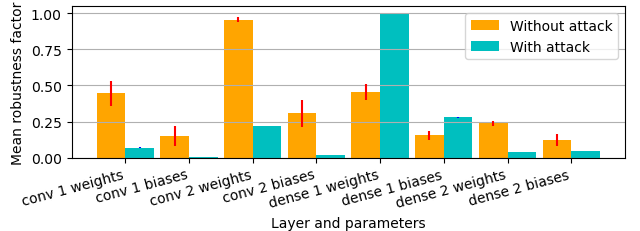}
  \caption{\footnotesize Mean robustness factors assigned to each layer by LEGATO across rounds 20-80 of training in the  Gaussian attack with $\sigma=20$ setting  and the non-attack setting.  Error bars indicate the variance of factors.}
\label{fig:factors}
\vspace{-.1in}
\end{figure}
% \todo{the above figure is a bit confusing. Why dense 2 doesn't have error bar in ``with attach'' setting? For conv 1 biases, why its variance is $\le 0$ in ``with attach'' setting? }

\begin{figure}
  \centering
  \begin{subfigure}{\columnwidth}
  \centering
  \includegraphics[width=\linewidth]{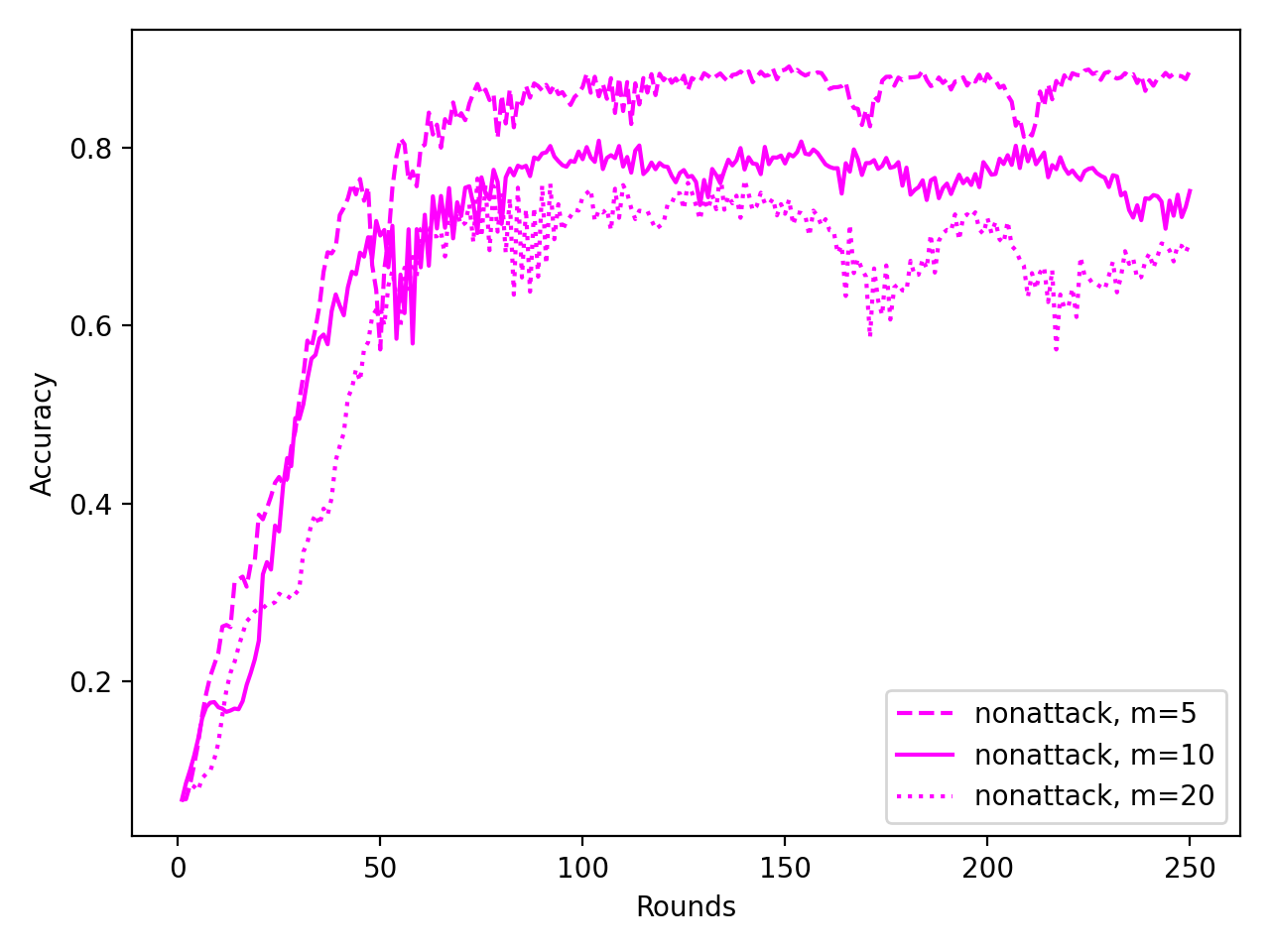}
  \subcaption{Non-attack setting.}
  \end{subfigure}
  \begin{subfigure}{\columnwidth}
  \includegraphics[width=\linewidth]{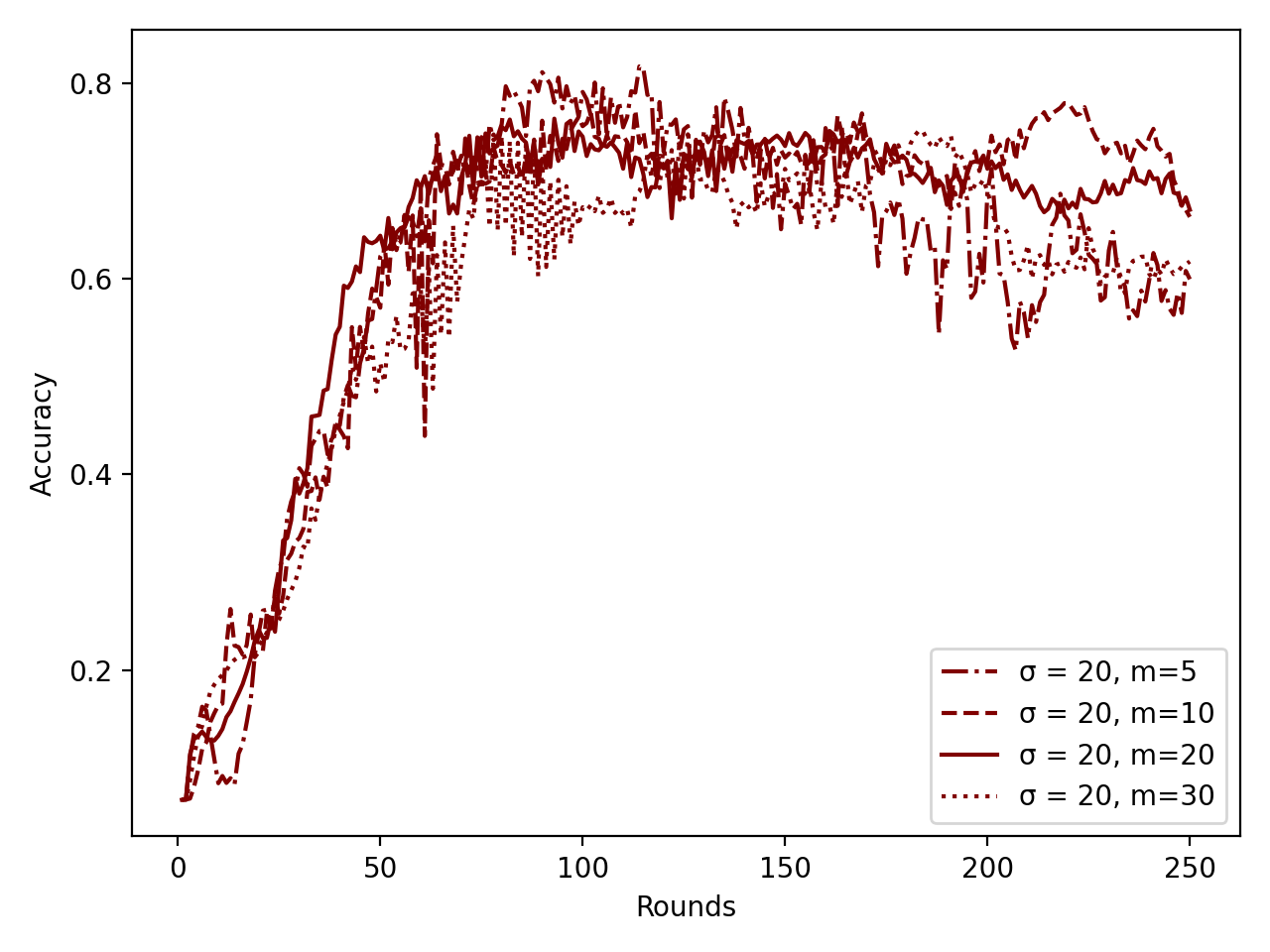}
  \subcaption{Gaussian attack, $\sigma=20$.}
  \end{subfigure}
  \begin{subfigure}{\columnwidth}
  \includegraphics[width=\linewidth]{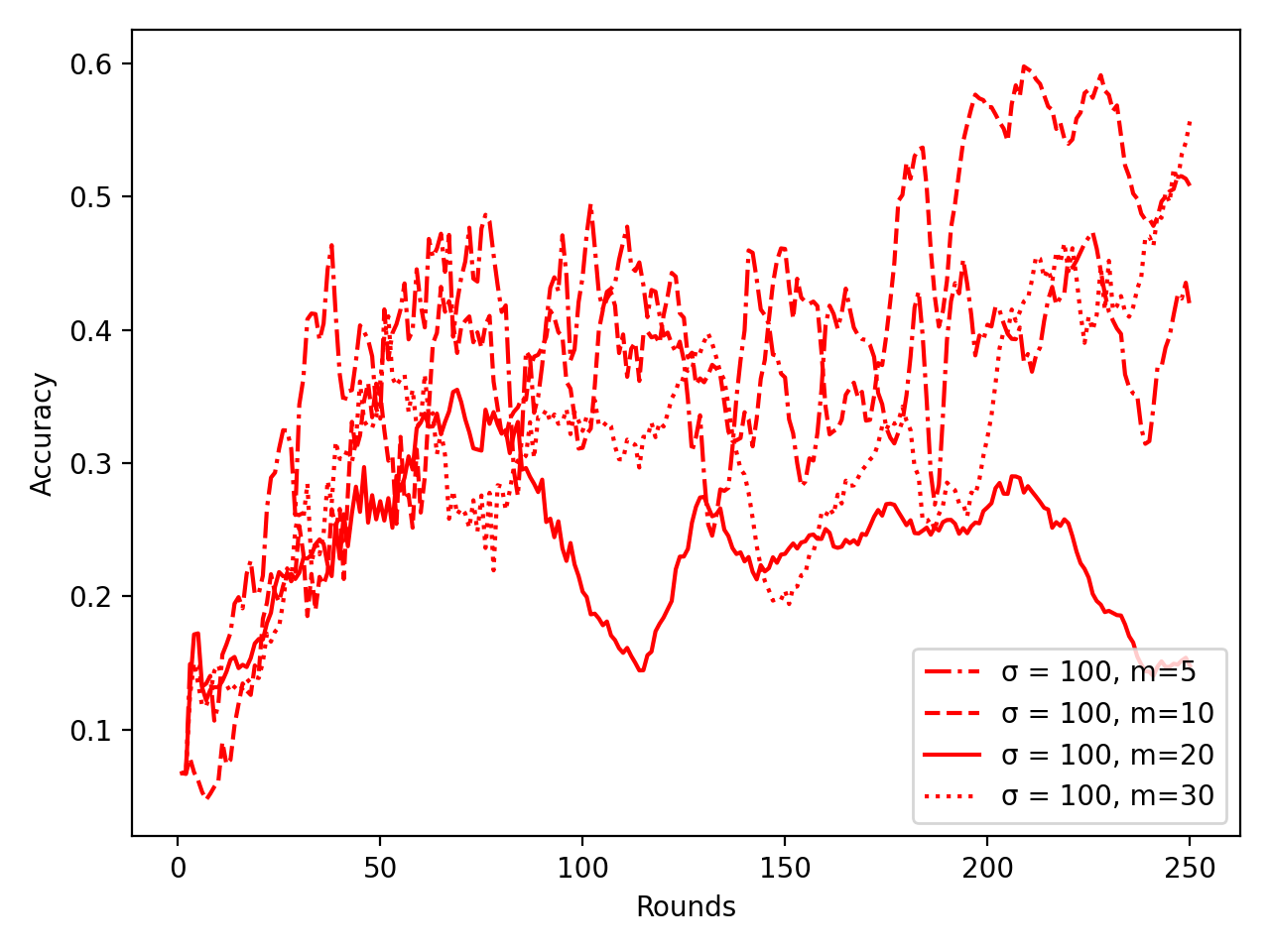}
  \subcaption{Gaussian attack, $\sigma=100$.}
  \end{subfigure}
  \caption{MNIST dataset, non-IID setting, where each worker has $1,000$ training images. 
  For Gaussian attacks, $4$ out of the $25$ workers are Byzantine and randomly sample gradients from $N(0, \sigma*I)$. 
  For all cases, the batch size is $50$ and the learning rate is $0.03$. We vary the log size from $5$ to $30$.}
  \label{fig:noniid_log}
\end{figure}

Figure~\ref{fig:factors} shows that in both settings, weights layers are more robust than their corresponding bias layers.
This pattern is more decisive in the Byzantine setting and the smaller variances associated with the factors from the Byzantine setting than the non-attack setting supports this claim.  A second pattern emerges with the Byzantine factors that shows later layers being generally more robust than earlier layers.  These two patterns are consistent with the observations from the pilot study.  They also align with general knowledge of layer gradient behavior in gradient descent.  Namely, that gradients vary more in the bias direction than in the weights direction and back-propagation more aggressively modifies earlier layers than later layers.  This again supports the idea introduced in the pilot study that Byzantine workers exploit the inherent vulnerability of each layer.

It is worth noting that, intuitively, layers with more parameters will be less robust due to their greater susceptibility to learning the noisy Byzantine gradients.  However, the $\ell_2$ variance-based quantification of robustness does not result in a ranking of layers that matches the ranking of their trainable parameter counts since conv1 has less trainable parameters than conv2.
This implies that while an overall increase in model parameters increases a model's overall susceptibility to noise, it is disproportionately concentrated in particular nodes or layers, irrespective of the layer dimensions.

%\subsubsection{Adding Outlier Detection.}

\subsection{Effectiveness of the log size}
Finally we examine the effect of the log size for the LEGATO algorithm, in particular, how does the log size affect the LEGATO's ability to defense against Byzatine threats.

From Figure~\ref{fig:noniid_log}, we can see that the choice of log size does not affect the model's performance a lot in the non-attack and Gaussian attack with small variance cases, but in the case of Gaussian attack with high variance, the selection of log size is quite crucial. 
In fact, it is reasonable that when log size is as small as five, the performance of LEGATO will lean towards gradient averaging, while as the log size increases, the LEGATO algorithm becomes more conservative to make a move towards the new gradient direction. Therefore, when the log size is as large as $30$, it will converge very slowly and even oscillate around certain point for a long time.

\section{Discussion}\label{sec:remark}
In this section, we discuss two limitations of LEGATO. First, LEGATO is limited to only neural networks. Second, the variance of the Gaussian attack can affect the effectiveness of LEGATO.

 Although LEGATO is a first successful attempt to mitigate Byzantine threats for training neural networks in federated learning, it is not a general algorithm that can be applied to train all types of machine learning models in FL setting. This is because the novelty of LEGATO is exploiting the characteristic of neural networks, i.e., utilizing the reaction of neural network layer architecture to identify robust layers and hence relying more heavily on those robust layers. 
However, other models like linear models, Support Vector Machines (SVMs) and XGBoost among others, do not possess such architecture. 
%One open problem in this field will be, 
Investigating if there exists any general robust aggregation algorithm that can be applied to mitigating Byzantine threats even under the non-IID local data distribution setting is an open research problem.

Moreover, as we can see in the case when the variance of the Gaussian attack is very high, the effectiveness of LEGATO is affected, since it does not reject any gradient information and hence is vulnerable to extreme outliers. However, currently there is no good way to define and identify ``extreme outliers'' according to the problem context.  In particular, there is no procedure to define and identify those ``extreme outliers'' in a data-aware fasion, i.e., based on local data distributions from all the workers.

\section{Related Work}\label{sec:related_art}

Many existing methods for Byzantine attack mitigation can be classified as either robust statistics or anomaly detection.
Some Byzantine robust algorithms employ inherently robust statistics such as coordinate median, trimmed mean \cite{yin2018byzantinerobust}, and their variants \cite{xie2018generalized,krishnaswamy2017constant,chen2017distributed,charikar2016learning,pillutla2019robust}
.  Anomaly detection approaches include \cite{muozgonzlez2019byzantinerobust,ijcai2019-670,xie2018zeno,fung2018mitigating,DBLP:journals/corr/abs-1907-12205,blanchard2017byzantinetolerant,mhamdi2018hidden}.
They filter out gradients assumed Byzantine by some form of outlier detection under the assumption that a local gradient vector is dishonest if it is comparatively more distant than other local vectors are to each other.  Krum~\cite{blanchard2017byzantinetolerant} is a state-of-the-art robust aggregation algorithm that evades Byzantine attacks by having the server only use a single worker's gradients to update the global model.  This worker is chosen at every round as the most trustworthy by having local gradients that are most similar to all other local gradients in terms of $\ell_2$ norms.  MultiKrum~\cite{blanchard2017byzantinetolerant} and Bulyan~\cite{mhamdi2018hidden} use Krum to choose a small set of worker gradients to aggregate.  Since these filtering methods reject many potentially honest gradients, they often sacrifice performance in non-attack settings.  
Additionally, Krum and its variants have quadratic computational complexity. One weakness of all aforementioned solutions is their required bounds on the variance of honest gradients~\cite{elmhamdi2020distributed}, which has been shown unknown in practice and can be exploited through attack strategies such as the Fall of Empires attack~\cite{xie2019fall}.  These variance bounds are especially uncharacteristic of settings where local data is non-IID, which is often the case in FL because workers will collect their training data samples from different data sources.  In general, the use of non-IID local data hinders the global model's performance and has motivated a new line of research for this problem, e.g., \cite{li2018federated,gao2019privacy,wang2020federated} and the references therein.  Other robust aggregation algorithms include methods that use algorithmic redundancy\cite{chen2018draco, data2019data, yu2018lagrange, DBLP:journals/corr/abs-1907-12205}and robust clustering methods\cite{ghosh2019robust, gupta2017kmeansoutliers, 10.5555/1347082.1347173}.

Some algorithms have a similar variance-reduction goal as LEGATO, but LEGATO is novel in its layer-specific gradient consideration. 
\cite{fu2019attackresistant} proposes a robust aggregation algorithm that ultimately dampens gradients through a worker-level reweighing scheme.  Distributed Momentum\cite{elmhamdi2020distributed} uses an exponential moving average to reduce the variance-norm ratio of gradients before feeding them to an expensive robust gradient aggregation algorithm (ie. Krum or coordinate median).
\cite{zhao2018federated, rieger2020client} exploit variance reduction to improve FL models' performance in non-IID settings, though not considering Byzantine threats.  Kardam \cite{damaskinos2018asynchronous} employs a gradient reweighing scheme that specifically intends to improve convergence, while a filtering step tackles robustness.  FedMA \cite{wang2020federated} works at a layer level, but it optimizes for permutation invariance of parameters in CNNs and LSTMs rather than for robustness.  \cite{zhang2019layers} demonstrates the layer-dependency of robustness, but solely in a centralized settings.

\section{Concluding Remarks}\label{sec:conclude}
LEGATO's novelty comes from its dynamic computation of layer-wise robustness factors for reweighing gradients when training neural networks in a FL setting.  Our empirical results show that LEGATO is robust across a wide range of Byzantine attack settings and does not reduce the performance of the model in absence of failures or attack.
In this paper, LEGATO specifically exploits the reaction of convolutional neural network layer architecture to identify robust layers and relies more heavily on them.  Other neural network structures can plausibly benefit from an adapted versions of LEGATO where the grouping of gradients is adjusted.  For instance, instead of grouping and reweighing gradients by layer, recurrent and residual neural networks could group by recurrent or residual blocks.
However, LEGATO is not applicable for use outside of neural networks.
Additionally, LEGATO is vulnerable to extreme outliers that come from high standard deviation in Gaussian attacks because it does not reject any gradient information.  However, we believe that such extremity is of lesser concern due to its blatancy.
In conclusion, LEGATO's is a general, scalable, efficient solution for Byzantine tolerance in FL that outperforms several state-of-the-art solutions in practical settings. 

\bibliographystyle{IEEEtran}
\bibliography{ref}

\end{document}